\documentclass[11pt]{article} 
\usepackage{rldmsubmit,palatino}
\usepackage{graphicx}
\usepackage{amsmath}
\usepackage{amsthm}
\usepackage{algorithm}
\usepackage{algpseudocode}
\usepackage{bm}
\usepackage{pgfplots}
\usepackage{url}
\pgfplotsset{
  compat=1.3,
}
\usepackage{amsfonts}       
\usepackage{tikz}
\usetikzlibrary{automata, positioning, arrows}
\tikzset{
    ->, 
    node distance=3cm, 
    every state/.style={thick}, 
}
\usepackage{subcaption}
\usepackage{adjustbox}

\newcommand{\mypara}[1]{\vspace{0pt}\noindent\textbf{#1}~~~}
\newtheorem{definition}{Definition}

\newtheorem{proposition}{Proposition}
\newtheorem{corollary}{Corollary}

\title{On the Expressivity of Multidimensional Markov Reward}

\author{
Shuwa Miura\\
Manning College of Information and Computer Sciences\\
University of Massachusetts Amherst\\
Amherst, MA 01003 \\
\texttt{smiura@cs.umass.edu} \\
}

\begin{document}

\maketitle

\begin{abstract}
We consider the expressivity of Markov rewards in sequential decision making under uncertainty.
We view reward functions in Markov Decision Processes (MDPs) as a means to characterize desired behaviors of agents.
Assuming desired behaviors are specified as a set of acceptable policies, we investigate if there exists a scalar or multidimensional Markov reward function that makes the policies in the set more desirable than the other policies.
Our main result states both necessary and sufficient conditions for the existence of such reward functions. 
We also show that for every non-degenerate set of deterministic policies, 
there exists a multidimensional Markov reward function that characterizes it.
\end{abstract}


\acknowledgements{
This research was supported by the National Science Foundation grants number IIS-1724101 and IIS-1954782.
We would like to thank Connor Basich, Saaduddin Mahmud, and Shlomo Zilberstein at University of Massachusetts Amherst for their  helpful comments and feedback on this work.
}

\startmain 

\section{Introduction}
Reward functions direct behaviors of agents. In some problems such as games, rewards are available for agents by default (e.g. as scores). For other problems, however, there are no reward signals inherently available in the environment. Rather, rewards are designed by the designer of the agent as a means to produce desired behaviors.

The vast majority of previous work on sequential decision making under uncertainty has focused on MDPs with scalar Markov rewards. 
Similarly, most Inverse Reinforcement Learning (IRL) \cite{abbeel_apprenticeship_2004} techniques learn scalar Markov reward functions from behaviors.
The implicit assumption is that scalar Markov rewards are sufficient to characterize desired behaviors of agents (Reward Hypothesis \cite{sutton_reward_2004}).

On the other hand, many real-world problems often contain multiple competing objectives and seem to require multidimensional reward functions \cite{roijers_multi-objective_2017}.
For example, an autonomous vehicle might want to minimize both the time to reach its destination and fuel consumption.
While the problem can intuitively be formulated using a two-dimensional reward function, there is little formal study on exactly when we need more than scalar Markov rewards.
One could still argue that a carefully engineered scalar reward function can represent the desired behaviors (e.g. via linear scalarization), but it is not clear when scalarization is possible.

This paper provides an initial direction to study expressivity of multidimensional Markov reward. 
Previous work \cite{abel_expressivity_2021} recently pointed out that some set of desired behaviors in Markov environment cannot be characterized by scalar Markov rewards.
Our work extends
the previous work \cite{abel_expressivity_2021} and asks the question: given a Markov environment and a specification for desired behaviors, does there exist a multidimensional Markov reward function that characterizes the desired behaviors?
Our main result states both necessary and sufficient conditions for the existence of scalar and multidimensional Markov rewards that characterize a given set of acceptable policies (Proposition 1 and 2). 
 We also show that for every non-degenerate set of deterministic policies, 
there exists a multidimensional Markov reward function that characterizes it (Corollary 1).

\section{Background}
\mypara{MDP} A finite MDP is a tuple $E=\langle S, A, T, \gamma, s_0, R \rangle$ where: $S$ is finite a set of states, $A$ is a finite set of actions, $T: S \times A \times S \rightarrow [0, 1]$ is a transition function, $\gamma$ is a discount factor, $s_0$ is the initial state, and $R: S \times A \rightarrow \mathbb{R}$ is a reward function.
We use $\bm{r}$ to denote a vector of size $(|S| \times |A|)$ representing $R$.
We refer to an MDP without a reward function as a Markov environment
($E=\langle S, A, T, \gamma, s_0 \rangle$).
A solution to an MDP is a \emph{policy}.
In this paper, We only consider \emph{stationary policies}, policies that do not change over time.
A \emph{deterministic policy} $\pi$ maps a state $s$ to an action $a \in A$.
A \emph{stochastic policy} $\pi$ maps a state $s$ to a probability distribution over $A$.
A policy $\pi$ induces a value function $V^{\pi}(s) = \mathbb{E}[\sum_{t=0}^{\infty} \gamma^t R(S_t, A_t)| S_0=s, \pi]$, which represents the expected discounted sum of rewards from $s_0$ by following $\pi$.
An \emph{optimal policy} $\pi^{*}$ is a policy that maximizes $V^{\pi}(s_0)$.
A \emph{discounted expected state-action visitation} of a policy $\pi$ is defined as $\rho^{\pi}(s, a) = \mathbb{E}[\sum_{t=0}^{\infty} \gamma^t Pr(S_t=s, A_t=a)| S_0=s_0, \pi]$.
We use $\bm{\rho}^{\pi}$ to denote a vector of 
$\rho^{\pi}(s, a)$.
We have $V^{\pi}(s_0) = \sum_{s \in S, a \in A} R(s,a) \rho^{\pi}(s, a)=\bm{r}^\top \bm{\rho}^{\pi}$.
$(1-\gamma) \bm{\rho}^{\pi}$
is called an \emph{occupancy measure}.

\mypara{Reward Design Problem} To formally investigate the expressivity of various reward functions, we ask questions of the form: 
given a Markov environment and a specification for desired behaviors, does there exist a Markov reward function that characterizes the desired behaviors? Previous work \cite{abel_expressivity_2021} explored reward design problems for different forms of specifications for desired behaviors.
Among the different specifications considered in the previous work, we focus on the case where a finite set of acceptable (good) policies and unacceptable (bad) policies ($\langle \Pi_G, \Pi_B \rangle$) are given. We require that $\Pi_G$ and $\Pi_B$ are disjoint, and call the pair SOAP as in the original paper \cite{abel_expressivity_2021}.
We say a scalar Markov reward function $R$ realizes $\langle E, \Pi_G, \Pi_B \rangle$ if every policy in $\Pi_{G}$ is optimal while every policy in $\Pi_B$ is not.

\begin{definition}[Scalar Reward Design Problem for SOAP]
  Given a Markov environment $E$ and SOAP $\langle \Pi_G, \Pi_B \rangle$, is there a reward function $R: S \times A \rightarrow \mathbb{R}$ such that $R$ realizes $\langle E, \Pi_G, \Pi_B \rangle$?
  \label{def:soap}
\end{definition}

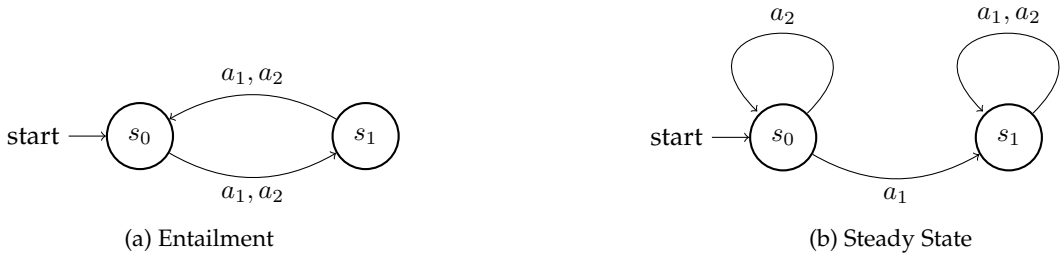
\begin{figure}[ht] 
    \centering 
    \begin{subfigure}[b]{0.49\linewidth}
        \centering
      \begin{tikzpicture}
          \node[initial, state] (q1) {$s_0$};
          \node[state, right of=q1] (q2) {$s_1$};
          \draw 
          (q2) edge[bend right, above] node{$a_1, a_2$} (q1)
          (q1) edge[bend right, below] node{$a_1, a_2$} (q2);
      \end{tikzpicture}
      \caption{Entailment}
      \label{fig:xor}
    \end{subfigure}
    \begin{subfigure}[b]{0.49\linewidth}
        \centering
      \begin{tikzpicture}
          \node[initial, state] (q1) {$s_0$};
          \node[state, right of=q1] (q2) {$s_1$};
          \draw 
          (q1) edge[loop, above] node{$a_2$} (q1)
          (q2) edge[loop, above] node{$a_1, a_2$} (q2)
          (q1) edge[bend right, below] node{$a_1$} (q2);
      \end{tikzpicture}
      \caption{Steady State}
\label{fig:degenerate}
    \end{subfigure}
      \caption{Examples of Markov environments from previous work \cite{abel_expressivity_2021}. }
    \label{fig:randomization}

\end{figure}

For example, consider the environment in Figure~\ref{fig:xor}. Let $\pi_{ij}$ be a deterministic policy taking $a_i$ at $s_0$ and $a_j$ at $s_1$. When $\Pi_G = \{ \pi_{11}\}$ and $\Pi_B= \{\pi_{12}, \pi_{21}, \pi_{22}\}$ (taking $a_1$ is always better), a reward function that sets $R(s_0, a_1)$ and $R(s_1, a_1)$ high realizes the given SOAP.
On the other hand, previous work \cite{abel_expressivity_2021} pointed out that when 
$\Pi_G = \{ \pi_{12}, \pi_{21}\}$ and $\Pi_B = \{ \pi_{11}, \pi_{22}\}$, there exists no scalar Markov reward function that characterizes the set of acceptable policies.
To make both 
$\pi_{12}$ and $\pi_{21}$ optimal, we would have to make 
$\pi_{11}$ and $\pi_{22}$ optimal as well.

Figure~\ref{fig:degenerate} shows another example where scalar Markov reward is not enough.
When $\pi_{21} \in \Pi_G$ and $\pi_{22} \in \Pi_B$, as the agent never visits $s_1$, $\pi_{21}$ cannot be better than $\pi_{22}$. 

Note that Definition~\ref{def:soap} is a generalization of the original definition \cite{abel_expressivity_2021}. The original definition assumes $\Pi_{B} = \overline{\Pi_{G}}$ and policies are deterministic.
Previous work \cite{abel_expressivity_2021}
proposed an LP formulation to solve the scalar reward design problem.

\section{Reward Design Problems for Multidimensional Markov Rewards}
While the previous work focused on scalar Markov reward setting, we now consider characterizing a set of acceptable policies using multidimensional Markov rewards.
First, we define consistency of SOAP to exclude degenerate problem instances like the one in Figure~\ref{fig:degenerate}.

\begin{definition}[Consistency of SOAP]
  A SOAP $\langle \Pi_G, \Pi_B \rangle$ is consistent iff for all $\pi_a \in \Pi_{G}$ and $\pi_b \in \Pi_{B}$, $\bm{\rho}^{\pi_a} \neq \bm{\rho}^{\pi_b}$.
\end{definition}

In other words, SOAP is consistent when policies with the same discounted expected state-action visitations are preferred equally.
For example, in Figure~\ref{fig:degenerate}, as $\bm{\rho}^{\pi_{21}} = \bm{\rho}^{\pi_{22}}$ , $\langle \Pi_G, \Pi_B \rangle =\langle \{\pi_{21}\}, \{\pi_{11}, \pi_{12}, \pi_{22}\} \}\rangle$ is not consistent.

With the presence of multiple objectives, there could be different ways to characterize acceptable policies.
One common approach is to use a set of Pareto-optimal policies \cite{roijers_multi-objective_2017}.
Another approach is to use Constrained Markov Decision Processes (CMDPs) \cite{altman_constrained_1999}, and characterize acceptable policies as those satisfying constraints.
In this paper, we use a constraint-based formulation for the ease of analysis.
We leave the Pareto-based approach to future work.

\begin{definition}
  Given $\langle E, R, \mathbf{c}\rangle$ where 
  $E$ is a Markov environment,
$\mathbf{c} \in \mathbb{R}^d$ is a lower-bound vector, and $R: S \times A \rightarrow \mathbb{R}^d$ is a $d$-dimensional reward function,
a policy $\pi$ is feasible for 
$\langle E, R, \mathbf{c}\rangle$ iff 
for all $i \in [d]$, $V^{\pi}_i(s_0) \geq c_i$.
\end{definition}

In other words, feasibility is defined as in CMDPs. But as we are only interested in characterizing a set of acceptable policies, we do not consider the maximization aspect of CMDPs.
For example, consider again the environment in Figure~\ref{fig:xor}.
Let $R(s_0,a_1) = R(s_1,a_1)= \begin{bmatrix}
0\\
0
\end{bmatrix}
$, 
$R(s_0,a_2) = R(s_1, a_2) = \begin{bmatrix}
1\\
-1
\end{bmatrix}
$, and
$\mathbf{c} = \begin{bmatrix}
2\\
-8
\end{bmatrix}
$.
Then as in Figure~\ref{fig:xor_plot}, $\pi_{12}$ and $\pi_{21}$ are feasible policies while $\pi_{11}$ and $\pi_{22}$ are infeasible.
We say $\langle E, R, \mathbf{c}\rangle$ \emph{realizes} $\langle \Pi_G, \Pi_B \rangle$ if every policy in $\Pi_G$ is feasible and every policy in $\Pi_B$ is infeasible.
The example shows that $\langle \Pi_G, \Pi_B \rangle = \langle \{\pi_{12}, \pi_{21}\}, \{\pi_{11}, \pi_{22}\}\rangle$ can be realized with $d=2$.
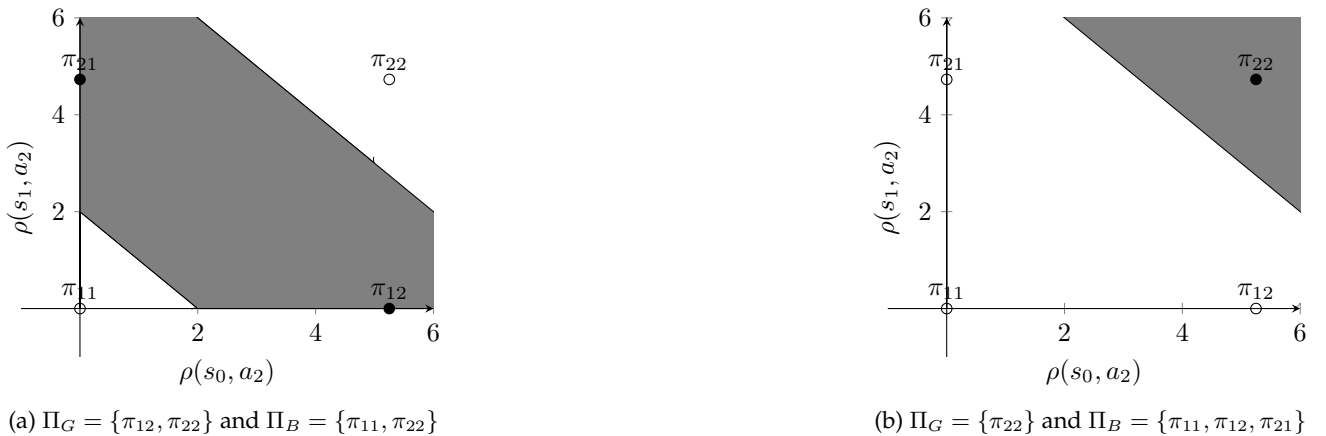
\begin{figure}[h]
  \centering
\begin{subfigure}{.38\linewidth}
      \centering
  \begin{tikzpicture}
      \begin{axis}[
        width=\linewidth,
        xmin = -1,
        xmax = 6,
        ymin = -1,
        ymax = 6,
        axis lines = center,
        area style,
        xlabel = {\(\rho(s_0, a_2)\)},
        ylabel = {\(\rho(s_1, a_2)\)},
        ylabel near ticks,
        xlabel near ticks,
    ]

    \addplot[
      color=black,
      ] coordinates{(0, 2) (2, 0)};

    \addplot[
      color=black,
      ]
      {
      8-x
      };
      
    \addplot[fill=gray] coordinates
    {(0,2) (2,0) (8,0) (0,8)}
    \closedcycle;  

    \addplot+ [nodes near coords,only marks,
    point meta=explicit symbolic, color=black, mark=o]
    table [meta=label] {
    x y label
    0 0 $\pi_{11}$
    5.25 4.73 $\pi_{22}$
    };
    
    \addplot+ [nodes near coords,only marks,
    point meta=explicit symbolic, color=black, mark=*]
    table [meta=label] {
    x y label
    0 4.73 $\pi_{21}$
    5.25 0.0 $\pi_{12}$
    };
    
      \end{axis}
  \end{tikzpicture}
  \caption{$\Pi_G = \{\pi_{12}, \pi_{22}\}$ and $\Pi_B = \{\pi_{11}, \pi_{22}\}$ }
  \label{fig:xor_plot}
\end{subfigure}%
  \hfill%
\begin{subfigure}{.38\linewidth}
      \centering
  \begin{tikzpicture}
      \begin{axis}[
        width=\linewidth,
        xmin = -1,
        xmax = 6,
        ymin = -1,
        ymax = 6,
        axis lines = center,
        area style,
        xlabel = {\(\rho(s_0, a_2)\)},
        ylabel = {\(\rho(s_1, a_2)\)},
        ylabel near ticks,
        xlabel near ticks,
    ]

    \addplot[fill=gray] coordinates
    {(0,50) (50,0) (8,0) (0,8)}
    \closedcycle;  

    \addplot+ [nodes near coords,only marks,
    point meta=explicit symbolic, color=black, mark=*]
    table [meta=label] {
    x y label
    5.25 4.73 $\pi_{22}$
    };
    
    \addplot+ [nodes near coords,only marks,
    point meta=explicit symbolic, color=black, mark=o]
    table [meta=label] {
    x y label
    0 0 $\pi_{11}$
    0 4.73 $\pi_{21}$
    5.25 0.0 $\pi_{12}$
    };
    
      \end{axis}
  \end{tikzpicture}
  \caption{$\Pi_G = \{\pi_{22}\}$ and $\Pi_B = \{\pi_{11}, \pi_{12}, \pi_{21}\}$ }
  \label{fig:lin_separatable}
\end{subfigure}
\caption{Visualization of policies for Figure~\ref{fig:xor} in discounted expected state-action visitation space with $\gamma=0.9$. Black dots represent policies in $\Pi_G$. White dots represent policies in $\Pi_B$. The grey area represents the feasible region.}
\end{figure}

The problem we consider then is:
\begin{definition}[Multidimensional Reward Design Problem for SOAP]
  Given a Markov environment $E$ and SOAP $\langle \Pi_G, \Pi_B \rangle$, is there $d \in \mathbb{N}$, 
$\mathbf{c} \in \mathbb{R}^d$, and $R: S \times A \rightarrow \mathbb{R}^d$ 
such that $\langle R, \mathbf{c} \rangle$ realizes $\langle E, \Pi_G, \Pi_B \rangle$?
  \label{def:multi_soap}
\end{definition}

We now argue that multidimensional reward design problems are equivalent to \emph{polyhedral separability} problems \cite{megiddo_complexity_1988}.
Let $\bm{r}_i$ be a vector of size $(|S| \times |A|)$ representing $i$-th dimensional rewards and  $\bm{R} = \begin{bmatrix}
  \bm{r}_1^\top \\
  \vdots \\
  \bm{r}_d^\top \\
\end{bmatrix}$. Then as $V^{\pi}_i(s_0) = \ \bm{r}_i^\top \bm{\rho}^{\pi}$, a policy $\pi$ is feasible iff $\bm{R}\bm{\rho}^{\pi} \geq \bm{c}$.
In other words, a policy $\pi$ is feasible iff it is within the polyhedron $\{\bm{x} \in \mathbb{R}^{|S|\times |A|}| \bm{R}\bm{x} \geq \bm{c}\}$, where each dimension of the reward function defines a hyperplane.
Let $\mathrm{P}$ be the set of discounted expected state-action visitation corresponding to a set of policies $\Pi$.
Now given $\langle \Pi_G, \Pi_B \rangle$, asking if there exists $d$-dimensional 
$\langle R, \mathbf{c}\rangle$ that realizes
$\langle E, \Pi_G, \Pi_B \rangle$ is equivalent to asking if there exists a polyhedron using $d$ hyperplanes that separates $\mathrm{P}_G$ and $\mathrm{P}_B$.
We denote the convex hull of $\mathrm{P}$ with $k$ elements as 
 $\textup{conv}(\mathrm{P}) = \{\sum_{i=1}^k \lambda_i \bm{\rho}_i| \bm{\rho}_i \in \mathrm{P}, \sum_{i=1}^k \lambda_i = 1, \lambda_i \in \mathbb{R}_{\geq 0} \}$.
From Proposition 2.2 on polyhedral separability in \cite{astorino_polyhedral_2002}, we immediately get the following result:

\begin{proposition}
  Given a Markov environment $E$ and
  SOAP $\langle \Pi_G, \Pi_B \rangle$, 
  there exists $d \leq |\Pi_B|$,
$\mathbf{c} \in \mathbb{R}^d$, and $R: S \times A \rightarrow \mathbb{R}^d$
that realizes $\langle E, \Pi_G, \Pi_B \rangle$ iff $\textup{conv}(\mathrm{P}_G) \cap \mathrm{P}_B = \emptyset$.
\end{proposition}

Intuitively, if we let $d = |\Pi_{B}|$, for each bad policy $\pi_b$, as $\bm{\rho}^{\pi_b} \not \in \textup{conv}(\mathrm{P}_G)$, we can pick a hyperplane that separates $\bm{\rho}^{\pi_b}$ from $\mathrm{P}_G$.
Note that $d = |\Pi_B|$ is not necessarily the smallest $d$ that realizes a given SOAP.
It is possible that one hyperplane separates multiple bad policies at once.
For example, Figure~\ref{fig:lin_separatable} shows a case where $|\Pi_B|=3$ but $d=1$ is enough to separate $\mathrm{P}_G$ from $\mathrm{P}_B$.
We now consider a special case where all policies are deterministic:

\begin{corollary}
  Given a Markov environment $E$ and
  consistent SOAP $\langle \Pi_G, \Pi_B \rangle$, 
  if all policies in $\langle \Pi_G, \Pi_B \rangle$ are deterministic, there exists $d \leq |\Pi_B|$,
$\mathbf{c} \in \mathbb{R}^d$, and $R: S \times A \rightarrow \mathbb{R}^d$
that realizes $\langle E, \Pi_G, \Pi_B \rangle$.
\end{corollary}

\begin{proof}[Proof Sketch]
Consistency of the SOAP ensures that 
$\mathrm{P}_G$ and $\mathrm{P}_B$ are disjoint. 
Let $\mathrm{P}$ be the set of all possible discounted expected state-action visitations for the given environment.
As each $\pi_b \in \Pi_B$ is a deterministic policy, $\bm{\rho}^{\pi_b}$ is an extreme point of $\mathrm{P}$.
That is, 
$\bm{\rho}^{\pi_b}$ cannot be written as a convex combination of other points in $\mathrm{P}$.
So for each $\pi_b \in \Pi_B$, we have $\bm{\rho}^{\pi_b} \not \in \textup{conv}(\mathrm{P}_G)$, which implies $\textup{conv}(\mathrm{P}_G) \cap \mathrm{P}_B = \emptyset$.
\end{proof}

An interesting question is whether there exists a scalar Markov reward function ($d=1$) that realizes a given SOAP.
In other words, we are asking if there exists a hyperplane that separates $\mathrm{P}_G$ from $\mathrm{P}_B$.
From Theorem 3 in \cite{mangasarian_linear_1965}, we get: 
\begin{proposition}
Given a Markov environment $E$ and
SOAP $\langle \Pi_G, \Pi_B \rangle$, 
there exists 
$c \in \mathbb{R}$, and $R: S \times A \rightarrow \mathbb{R}$
that realizes $\langle E, \Pi_G, \Pi_B \rangle$
iff
$\textup{conv}(\mathrm{P}_G) \cap \textup{conv}(\mathrm{P}_B) = \emptyset$.
\end{proposition}
Note that the setting is slightly different from that of Definition~\ref{def:soap}.
Every policy $\pi_a \in \Pi_G$ must satisfy $V^{\pi_a}(s_0) \geq c$, but does not have to be optimal.
If $R$ makes every $\pi_a \in \Pi_G$ optimal, we can easily characterize $\Pi_G$
as a set of feasible policies with the same $R$ and $c=V^{*}(s_0)$.
However, the converse is not true in general\footnote{This is essentially difference between range-SOAP and equal-SOAP described in the previous work \cite{abel_expressivity_2021}}.

%
\section{Conclusions and Future Work}
We have presented our initial attempts to investigate the expressivity of multidimensional Markov rewards, extending the recent work \cite{abel_expressivity_2021}.
Our main result shows both necessary and sufficient conditions for the existence of scalar and multidimensional Markov rewards that realizes a given set of acceptable policies (Proposition 1 and 2).
We additionally show that a set of deterministic policies can always be characterized by some multidimensional Markov rewards (Corollary 1).

We foresee the following directions for future research:

\mypara{More General Task Specifications}
This paper assumes that desired behaviors of agents are specified as a set of acceptable policies.
A natural direction forward is to extend our results to general preferences among policies and trajectories.

\mypara{Algorithm and Complexity}
While our results state conditions under which multidimensional Markov rewards realize a given set of acceptable policies, we did not provide an algorithm to check the conditions.
We plan to develop the algorithm based on ideas from \cite{astorino_polyhedral_2002}
and analyze the computational complexity of the problem.

\mypara{Preference-based IRL using Multidimensional Rewards}
We plan to extend our idea to preference-based IRL \cite{wirth_survey_2017,brown_extrapolating_2019} using multidimensional rewards.
The binary classification problem we considered above can be thought of as a special case of preference learning, where the preference forms only two classes.
Our preference-based IRL approach would use a loss function inspired by separability of policies.

\mypara{Decision Theory}
Exploring expressivity of rewards has a strong connection to classical decision theory.
In particular, SOAPs considered in this work can thought of as
violating the Independence Axiom from the VNM Theorem \cite{von_neumann_theory_1944}.
We plan to explore the connection between bounded rationality of the designer and expressivity of Markov rewards.
There is some previous work \cite{sunehag_axioms_2011,pitis_rethinking_2019} that explored the connection between reinforcement learning and decision theory. 
\clearpage
\bibliographystyle{plain}
\bibliography{references}
\end{document}